\documentclass[letterpaper,10pt,conference]{ieeeconf}

\IEEEoverridecommandlockouts %
\usepackage{amsmath} %
\usepackage{amssymb} %
\usepackage{amsfonts}
\usepackage[sort,compress]{cite}
\usepackage[usenames,dvipsnames]{xcolor}
\usepackage{graphicx}
\usepackage{tabularx}
\usepackage{multirow}
\usepackage{mathtools}
\usepackage{esvect} %
\usepackage[binary-units=true]{siunitx}
\usepackage{caption}
\usepackage[pdftex, pdfstartview={FitV}, pdfpagelayout={TwoColumnLeft},bookmarksopen=true,plainpages = false, colorlinks=true, linkcolor=black, citecolor = black, urlcolor = black,filecolor=black , pagebackref=false,hypertexnames=false, plainpages=false, pdfpagelabels ]{hyperref}
\usepackage[T1]{fontenc} %
\usepackage{mathtools, cuted} %
\usepackage{bm}
\usepackage{xspace}

\usepackage{balance} %
\usepackage{booktabs} %
\usepackage[export]{adjustbox} %

\newcolumntype{x}[1]{%
>{\raggedleft\hspace{0pt}}p{#1}}%

\definecolor{scan1blue}{HTML}{0072bd}
\definecolor{scan2red}{HTML}{d95319}

\usepackage{algorithm}
\usepackage[noend]{algpseudocode} %
\definecolor{commentclr}{RGB}{34, 139, 34}

\usepackage{tikz}
\usetikzlibrary{plotmarks}

\newcommand{\tikzcircle}[2][red,fill=red]{\tikz[baseline=-0.5ex]\draw[#1,radius=#2] (0,0) circle ;}%

\usepackage{subcaption}
\DeclareCaptionLabelSeparator{periodspace}{.\quad}
\newcommand{\ra}[1]{\renewcommand{\arraystretch}{#1}}

\usepackage{amsthm}

\newtheorem{prop}{Proposition}

\theoremstyle{definition}

\newtheorem{definition}{Definition}

\makeatletter
\newcommand{\newreptheorem}[2]{%
\newtheorem*{rep@#1}{\rep@title}%
\newenvironment{rep#1}[1]{%
 \def\rep@title{#2 \ref*{##1}}%
 \begin{rep@#1}}%
 {\end{rep@#1}}}
\makeatother
\newreptheorem{prop}{Proposition}

\usepackage{letltxmacro}
\LetLtxMacro\orgvdots\vdots
\LetLtxMacro\orgddots\ddots

\makeatletter
\DeclareRobustCommand\vdots{%
	\mathpalette\@vdots{}%
}
\newcommand*{\@vdots}[2]{%
	\sbox0{$#1\cdotp\cdotp\cdotp\m@th$}%
	\sbox2{$#1.\m@th$}%
	\vbox{%
		\dimen@=\wd0 %
		\advance\dimen@ -3\ht2 %
		\kern.5\dimen@
		\dimen@=\wd2 %
		\advance\dimen@ -\ht2 %
		\dimen2=\wd0 %
		\advance\dimen2 -\dimen@
		\vbox to \dimen2{%
			\offinterlineskip
			\copy2 \vfill\copy2 \vfill\copy2 %
		}%
	}%
}
\DeclareRobustCommand\ddots{%
	\mathinner{%
		\mathpalette\@ddots{}%
		\mkern\thinmuskip
	}%
}
\newcommand*{\@ddots}[2]{%
	\sbox0{$#1\cdotp\cdotp\cdotp\m@th$}%
	\sbox2{$#1.\m@th$}%
	\vbox{%
		\dimen@=\wd0 %
		\advance\dimen@ -3\ht2 %
		\kern.5\dimen@
		\dimen@=\wd2 %
		\advance\dimen@ -\ht2 %
		\dimen2=\wd0 %
		\advance\dimen2 -\dimen@
		\vbox to \dimen2{%
			\offinterlineskip
			\hbox{$#1\mathpunct{.}\m@th$}%
			\vfill
			\hbox{$#1\mathpunct{\kern\wd2}\mathpunct{.}\m@th$}%
			\vfill
			\hbox{$#1\mathpunct{\kern\wd2}\mathpunct{\kern\wd2}\mathpunct{.}\m@th$}%
		}%
	}%
}
\makeatother

\let\oldnl\nl%
\newcommand{\nonl}{\renewcommand{\nl}{\let\nl\oldnl}}%
\makeatother

\def\sl{\mathcal{L}}

\def\Gr{\mathrm{Gr}}
\def\Graff{\mathrm{Graff}}
\def\bA{\mathbb{A}}
\def\bY{\mathbb{Y}}
\def\S{\mathcal{S}}

\graphicspath{{figures/}}

\usepackage{svg}
\svgpath{{figures/}}

\title{\LARGE \bf Global Data Association for SLAM with\\%
3D Grassmannian Manifold Objects}

\author{Parker C. Lusk and Jonathan P. How%
	\thanks{The authors are with the Department of Aeronautics and Astronautics, Massachusetts Institute of Technology.
	    {\texttt{\{plusk, jhow\}@mit.edu.}}}
    \thanks{This work is supported by the Ford Motor Company.}
}%

\begin{document}

\maketitle
\thispagestyle{plain}
\pagestyle{plain}

\begin{abstract} 
Using pole and plane objects in lidar SLAM can increase accuracy and decrease map storage requirements compared to commonly-used point cloud maps.
However, place recognition and geometric verification using these landmarks is challenging due to the requirement for global matching without an initial guess.
Existing works typically only leverage either pole or plane landmarks, limiting application to a restricted set of environments.
We present a global data association method for loop closure in lidar scans using 3D line and plane objects simultaneously and in a unified manner.
The main novelty of this paper is in the representation of line and plane objects extracted from lidar scans on the manifold of affine subspaces, known as the affine Grassmannian.
Line and plane correspondences are matched using our graph-based data association framework and subsequently registered in the least-squares sense.
Compared to pole-only approaches and plane-only approaches, our 3D affine Grassmannian method yields a \SI{71}{\percent} and \SI{325}{\percent} increase respectively to loop closure recall at \SI{100}{\percent} precision on the KITTI dataset and can provide frame alignment with less than \SI{10}{\cm} and \SI{1}{\deg} of error.
\end{abstract}

\section{Introduction}\label{sec:intro}

Geometric verification provides a critical line of defense against incorrect loop closure, which can lead to disastrous map distortion and estimation error.
Place recognition modules attempt to suggest previously explored areas that are similar to current local sensor observations, but require a geometric verification step to confirm the loop closure hypothesis and to provide a geometric constraint between the pair of associated poses.
These constraints are extremely valuable in reducing odometric drift present in simultaneous localization and mapping (SLAM) systems~\cite{cadena2016past}, so long as they are correct.
The core challenge of place recognition and geometric verification is associating current local observations with previously processed observations without relying on an initial guess.
This challenge is known as global data association~\cite{durrant2006simultaneous,bailey2006simultaneous} and is at the heart of many perception problems, such as extrinsic calibration, multi-robot map merging, loop closure detection, and global (re)localization.

In the visual place recognition~\cite{lowry2015visual} setting, image features are commonly used in bag-of-words techniques~\cite{galvez2012bags} for loop candidate retrieval and geometric verification.
However, appearance-based methods are sensitive to illumination, weather, and viewpoint changes and can fail to detect loop closures in these settings.
Alternatively, geometric-based methods~\cite{yin2018locnet,kim2018scan,chen2021auro} utilizing 3D lidar sensors are more resilient to these changes, but come at the expense of processing and storing hundreds of thousands of point measurements per second.
To maintain the benefits of geometric data, but to reduce the storage and computational costs of large point maps, some lidar odometry and SLAM systems use geometric primitives like lines and planes instead of points~\cite{brenner2009global,pathak2010online,schaefer2019long,kummerle2019accurate,cao2021lidar}.
In addition to providing lightweight maps with high-level semantic information, navigating using explicit planes extracted from the environment provides extra information over points and has lead to improved, low-drift odometry~\cite{kaess2015simultaneous,hsiao2017keyframe,geneva2018lips}.
In fact, even utilizing \emph{points} (momentarily ignoring the storage costs) that exhibit strong local planarity have allowed for high-quality lidar-based odometry systems~\cite{zhang2014loam,shan2018lego}.
While existing works either use lines/poles or planes (often in 2D) for global data association, a remaining challenge is performing global data association using 3D lines and planes simultaneously.
We present an efficient and robust method for global data association and geometric verification amenable to any combination of points, lines, and planes.

\begin{figure}[t]
    \centering
    \includegraphics[trim=1cm 1cm 1cm 2cm, clip, width=\columnwidth]{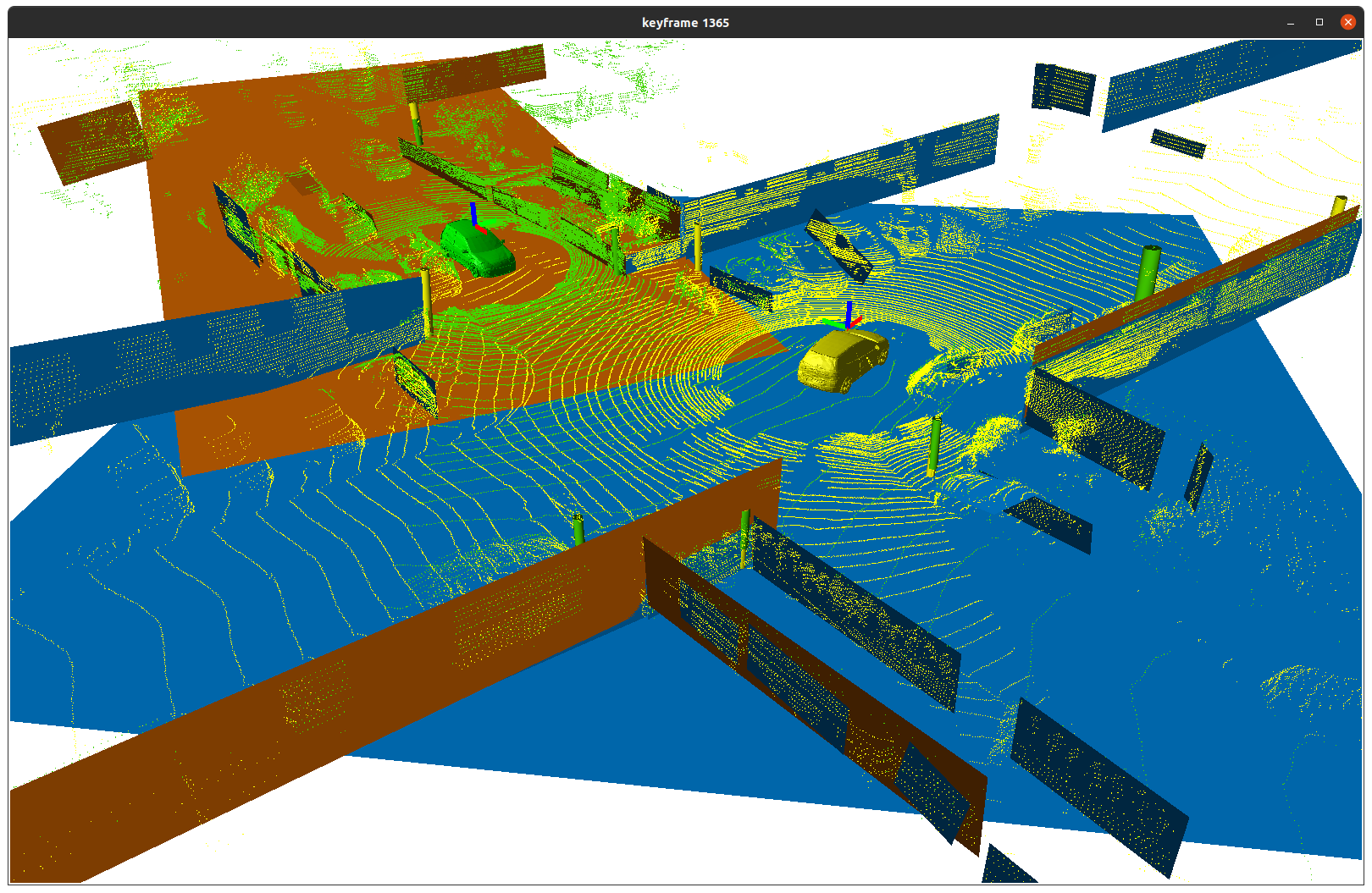}
    \caption{Successful alignment between the lidar scans of a loop closure hypothesis.
    Sensor origins are denoted by the green and yellow cars, which are \SI{18}{\m} apart.
    Poles and planes extracted from each lidar scan are represented as 3D affine Grassmannian elements.
    Using the associated Riemannian metric allows for the evaluation of geometric consistency between object pairs, even between a pole and a plane.
    Object correspondences with high pairwise consistency are identified using our graph-based global data association algorithm and then used to estimate the rigid transformation between the two frames, yielding an alignment error of \SI{4}{\cm} and \SI{0.3}{\deg}.
    }
    \label{fig:teaser-image}
\end{figure}

A key novelty of our approach is in the representation of line and plane landmarks as elements of a Grassmannian manifold, which is the space of all linear subspaces.
In particular, we utilize the \emph{affine} Grassmannian manifold, which allows for the representation of affine subspaces (i.e., linear subspaces not necessarily containing the origin).
By leveraging this manifold representation, distances between simple geometric landmarks can easily be defined in a principled manner.
We use these distances between pairwise landmarks in each lidar scan to build a consistency graph, enabling the use of our robust, graph-theoretic global data association framework~\cite{lusk2021clipper} to find the largest set of landmark associations that are geometrically consistent.
Then, the rigid transformation between a pair of candidate loop closure scans can be estimated by solving a line and plane registration problem with known correspondences in the least-squares sense.
Experimental evaluation of loop closure verification on the KITTI dataset~\cite{geiger2012we} shows that our method surpasses the state-of-the-art in global data association with geometric primitives. %
Compared to pole-only approaches and plane-only approaches, our method yields a \SI{71}{\percent} and \SI{325}{\percent} increase respectively to loop closure recall at \SI{100}{\percent} precision.
In summary, our main contributions are:
\begin{itemize}
    \item the introduction of the affine Grassmannian representation of pole and plane objects for global data association, leading to geometric verification with 3D landmarks free of requirements of an initial alignment;
    \item a least squares estimator for rigid transformation using lines and planes instead of points, leading to a more accurate estimate for rotation and translation;
    \item evaluation of loop closure geometric verification on four sequences of the KITTI~\cite{geiger2012we} dataset, showing superior recall and accuracy over the state-of-the-art.
\end{itemize}
We emphasize that this is the first work using the affine Grassmannian manifold for data association, which provides a unifying and principled framework for associating points, lines, planes (or higher dimensional linear objects) in robotic loop closure and geometric verification problems.

\section{Related Work}\label{sec:relatedwork}

\textbf{Scan Matching}.
The most basic element of lidar-based navigation is \emph{local} data association, often referred to as scan matching.
It is frequently carried out by iterative closest point (ICP)~\cite{besl1992method} and its variants~\cite{rusinkiewicz2001efficient}, though care must be taken to provide a good initialization, otherwise, a wrong odometry solution can be obtained.
Bosse and Zlot~\cite{bosse2009continuous} perform scan matching on spinning 2D lidar sweeps where the correspondence generation step of ICP is informed by local shape information.
LOAM~\cite{zhang2014loam} uses feature-based scan matching by minimizing the distance between edge points and planar points of subsequent scans in a Levenberg-Marquardt (LM) framework, resulting in high-rate, low-drift odometry.
LeGO-LOAM~\cite{shan2018lego} specializes LOAM for ground vehicles with limited computation; by first extracting ground points and segmenting remaining points into local clusters, noisy points can be removed and scan matching is performed in a two-step LM optimization.

\textbf{Place Recognition}.
Scan matching alone will introduce drift over time, which can be reduced via loop closure or localization within a map.
To identify potential loop closure scan pairs, some systems extract a compact global descriptor of the scan~\cite{yin2018locnet,kim2018scan} which is then used to search for similar scans via a k-d tree.
Once the top loop candidate is identified, the rigid transformation between two scans is refined using ICP, which requires that the initial pose error relating the two scans is low and that all the points be saved for each each scan.
Descriptors of a subset of the scan could instead be extracted~\cite{bosse2013place,cop2018delight} and subsequently matched, but handcrafted features are especially known to be unstable due to the sparsity of the lidar point cloud~\cite{dewan2018learning}.
SegMap~\cite{dube2020segmap} incrementally segments new scans into clusters of a local map to overcome the sparsity of scans and to reduce the number of points required to store, after which descriptors of each cluster are used to find matches, followed by a graph-based geometric verification step.

Other graph-based methods~\cite{zhu2020gosmatch,kong2020semantic} leverage semantic information to create histogram-based global descriptors used for place retrieval, followed by RANSAC~\cite{fischler1981random} geometric verification.
Fern\'andez-Moral et al.~\cite{fernandez2013fast} present a graph-based place recognition system which matches planar patches extracted from RGB-D scans using an interpretation tree to validate various unary and binary constraints between candidate plane matches, followed by geometric verification.
Jiang et al.~\cite{jiang2020lipmatch} extend~\cite{fernandez2013fast} and introduce additional unary and binary constraints.
Some of these constraints are sensitive to viewpoint changes, thus these methods rely on close proximity of the 3D scans as an initialization, precluding their applicability in settings like global localization.

Pathak et al.~\cite{pathak2010online} use 3D plane landmarks in a relaxed graph-based SLAM and perform data association of planes using a series of geometric tests followed by a maximum consensus selection~\cite{pathak2010fast}.
Kaess~\cite{kaess2015simultaneous} extracts 3D planes from RGB-D data and proposes a novel quaternion-based representation of planes for use in SLAM which avoids the issues of overparameterized state vector in nonlinear least-squares estimation.
Geneva et al.~\cite{geneva2018lips} alternatively introduce the ``closest point'' (CP) parameterization of planes for estimation and demonstrate its advantages in lidar-inertial SLAM.
However, \cite{kaess2015simultaneous} and \cite{geneva2018lips} do not provide a means for global data association for detection of loop closures.
Zhou et al.~\cite{zhou2021pi} present an indoor smoothing and mapping system which incorporates a plane-based bundle adjustment.
Loop closures candidates, identified by previous keyframes in close proximity, are verified by first matching plane CP vectors, followed by a plane-based RANSAC.

Pole-based localization methods~\cite{schaefer2019long,kummerle2019accurate,wilbers2019localization} commonly treat poles as 2D points based on their intersection with the ground plane and use point-based registration methods for geometric verification given an initial guess.
Brenner~\cite{brenner2009global} investigates the use of upright poles extracted from lidar scans to construct descriptors for global localization.
Schlichting and Brenner~\cite{schlichting2014localization} extend this descriptor to include upright planes, but effectively treat poles and planes as 2D points and lines.
Cao et al.~\cite{cao2021lidar} perform object-level SLAM using poles, walls, and parked cars as landmarks and propose to use pole positions within a scan to create a scan signature used for detecting loops.
Upon detecting a pair of scans as a loop candidate, clusters of pole-points are matched~\cite{cao2020accurate} and a rigid transformation is estimated in a RANSAC framework.

Our method similarly leverages poles and planes, but is not limited to treating these landmarks as 2D objects and does not make assumptions on the proximity of scans, nor does it require an initial alignment guess.
Instead, we perform global data association by identifying matches based on pairwise geometric consistency between lines and planes.
Thus, our method provides a means for obtaining the transformation between two sets of geometric objects, a key feature for place recognition in object-based maps.

\textbf{Grassmannian Manifold}.
The Grassmannian manifold has been used extensively in subspace learning~\cite{hamm2008grassmann}, especially in face recognition~\cite{huang2015projection} and appearance tracking~\cite{shirazi2014object} tasks in the computer vision community.
Rentmeesters et al.~\cite{rentmeesters2010filtering} develop an observer for subspace tracking on the manifold.
Calinon~\cite{calinon2020gaussians} outlines the use of Riemannian manifolds in robotics and notes the under-representation of the Grassmannian.

\subsection{Preliminaries}\label{sec:preliminaries}

We briefly introduce the Grassmannian manifold.
For a more comprehensive introduction, we refer the reader to~\cite{edelman1998geometry}.
The Grassmannian is the space of $k$-dimensional subspaces of $\mathbb{R}^n$, denoted $\Gr(k,n)$.
For example, $\Gr(1,3)$ represent 3D lines containing the origin.
An element $\bA\in\Gr(k,n)$ is represented by an orthonormal matrix $A\in\mathbb{R}^{n\times k}$ whose columns form an orthonormal basis of $\bA$.
Note that the choice of $A$ is not unique.
The geodesic distance between two subspaces $\bA_1\in\Gr(k_1,n)$ and $\bA_2\in\Gr(k_2,n)$ is
\begin{equation}
d_\mathrm{Gr}(\bA_1, \bA_2) = \left(\sum_{i=1}^{\min(k_1,k_2)} \theta_i^2\right)^{1/2}
\end{equation}
where $\theta_i$ are known as the principal angles~\cite{edelman1998geometry}.
These angles can be computed via the singular value decomposition (SVD) of the corresponding orthonormal matrices of $\bA_1$ and $\bA_2$,
\begin{equation}
A_1^\top A_2 = U\, \mathrm{diag}(\cos\theta_1, \dots, \cos\theta_k )\, V^\top.
\end{equation}
Note that if the subspaces are of unequal dimension, the number of principal angles is equal to the smaller dimension of the two.

\begin{figure}[t]
    \centering
    \begin{subfigure}[b]{0.49\columnwidth}
        \includeinkscape[pretex=\footnotesize,width=\columnwidth]{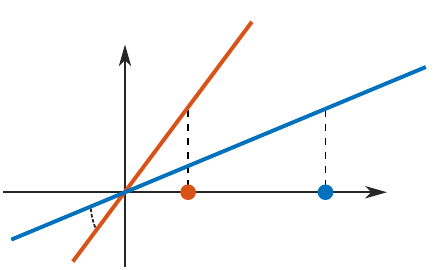}
        \caption{}
        \label{fig:graffexample}
    \end{subfigure}
    \begin{subfigure}[b]{0.49\columnwidth}
        \includegraphics[trim = 0mm 0mm 0mm 0mm, clip, width=1\linewidth]{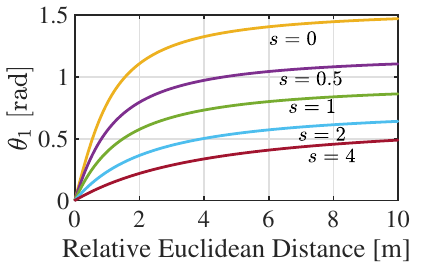}
        \caption{}
        \label{fig:graffsensitivity}
    \end{subfigure}
    \caption{(a) Example of a point in $\Graff(0,1)$ being embedded as a line in $\Gr(1,2)$.
    The principal angle between these two linear subspaces is $\theta_1$.
    (b) When applied directly, $d_\Graff$ is not invariant to global translation $s$.
    }
    \label{fig:graffexample-both}
\end{figure}

We are specifically interested in affine subspaces of $\mathbb{R}^3$, e.g., lines and planes that are at some distance away from the origin.
In analogy to $\mathrm{Gr}(k,n)$, the set of $k$-dimensional affine subspaces constitute a smooth manifold called the \emph{affine Grassmannian}, denoted $\Graff(k,n)$~\cite{lim2021grassmannian}.
We write an element of this manifold as $\bY=\bA+b\in\Graff(k,n)$ with affine coordinates $[A,b]\in\mathbb{R}^{n\times(k+1)}$, where $A\in\mathbb{R}^{n\times k}$ is an orthonormal matrix and $b\in\mathbb{R}^n$ is the displacement of $\bA$ from the origin.
We emphasize that $\Graff(k,n)\neq\Gr(k,n)\times\mathbb{R}^n$.
Instead, an element $\bY\in\Graff(k,n)$ is treated as a higher-order subspace via the embedding
\begin{align}
j:\Graff(k,n)&\hookrightarrow\Gr(k+1,n+1), \notag \\ 
\bA+b &\mapsto \mathrm{span}(\bA\cup\{b+e_{n+1}\}),
\end{align}
where $e_{n+1} = (0,\dots,0,1)^\top\in\mathbb{R}^{n+1}$ (see \cite[Theorem 1]{lim2021grassmannian}).
Fig.~\ref{fig:graffexample} shows an example of a point in $\mathbb{R}$ being embedded as a line in $\mathbb{R}^2$.

The Stiefel coordinates of $\bY\in\Graff(k,n)$,
\begin{equation}
Y =
\begin{bmatrix}
A & b_0/\sqrt{1+\|b_0\|^2} \\
0 & 1/\sqrt{1+\|b_0\|^2}
\end{bmatrix}\in\mathbb{R}^{(n+1)\times(k+1)},
\end{equation}
allow for the computation of distances between two affine subspaces using the Grassmannian metric,
\begin{equation}
d_\Graff(\bY_1,\bY_2) = d_\Gr(j(\bY_1),j(\bY_2)),
\end{equation}
with principal angles computed via the SVD of $Y_1^\top Y_2$.
The vector $b_0\in\mathbb{R}^n$ is the orthogonal displacement of $\bA$, which is the projection of $b$ onto the left nullspace of $A$ s.t. $A^\top b_0=0$.

For convenience, the line $\bY^\ell\in\Graff(1,3)$ may also be represented in point-direction form as $\ell = [A;b]\in\mathbb{R}^6$, and a plane $\bY^\pi\in\Graff(2,3)$ may be represented in Hesse normal form as $\pi = [n;d]\in\mathbb{R}^4$ where $n = \mathrm{ker}\,A^\top$ and $d = \|b_0\|$.
Under a rigid transformation $T=(R,t)\in\mathrm{SE}(3)$, the transformation law of lines and planes can be written
\begin{align}
\ell' &= f_\ell(\ell,R,t) := \begin{bmatrix}RA&Rb+t\end{bmatrix}^\top \\
\pi'  &= f_\pi(\pi,R,t) := T^{-\top}\pi.
\end{align}

\section{Method}\label{sec:method}

Given a candidate pair of lidar scans produced by, e.g., matching global scan descriptors~\cite{zhu2020gosmatch} or comparison with past keyframes~\cite{jiang2020lipmatch}, we seek to geometrically verify the loop pair and produce a relative transformation between the two sensor poses.
In the following discussion, we assume that scan $i$ has already had $l_i$ lines and $p_i$ planes extracted, and we refer to them collectively as objects $s_{i,a}\in{\S_i = \{ \bY^\ell_1,\dots,\bY^\ell_{l_i}, \bY^\pi_1,\dots,\bY^\pi_{p_i}\}}$.
Our method is comprised of the following steps: (i) constructing a consistency graph based on pairwise object distances in each scan, (ii) identifying object correspondences via the densest fully-connected subgraph in the consistency graph, and (iii) estimating a rigid transformation based on object correspondences.

\begin{figure}[t]
    \centering
    \includeinkscape[pretex=\footnotesize,width=\columnwidth]{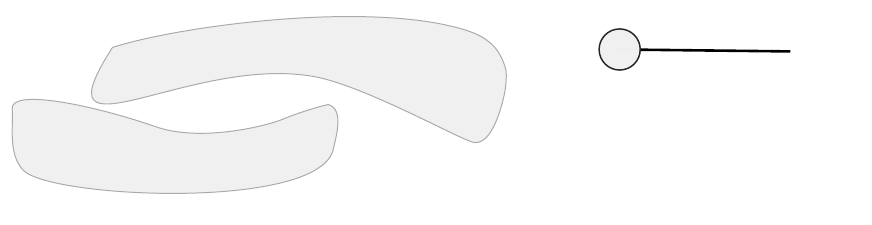}
    \caption{
    Construction of a consistency graph.
    Using $d_\Graff$, the distance between a line and a plane in scan $\S_i$ (\tikzcircle[scan1blue,fill=scan1blue]{1.5pt}) is compared to the distance between the two corresponding objects in $\S_j$ (\tikzcircle[scan2red,fill=scan2red]{1.5pt}).
    The consistency of these two distances is evaluated using \eqref{eq:consistency} and the edge $(u_1,u_2)$ is so weighted.
    }
    \label{fig:consistencygraph}
\end{figure}

\subsection{Consistency Graph Construction}\label{sec:consistency}

A consistency graph for two scans $\mathcal{S}_i$, $\mathcal{S}_j$ is an undirected weighted graph $\mathcal{G}=(\mathcal{V},\mathcal{E},w)$ with potential object correspondences $s_{i,a}\leftrightarrow s_{j,b}$ as vertices, edges between consistent correspondences, and a weighting function $w:\mathcal{E}\to[0,1]$ that evaluates the strength of consistency.
A pair of correspondences $u_1,u_2\in\mathcal{V}$ is consistent if the distance between the underlying objects $s_{i,a}\in\S_i,s_{j,b}\in\S_j$ satisfies
\begin{equation}\label{eq:consistency}
c_{u_1,u_2} = |d(s_{i,u_1^a},\,s_{i,u_2^a}) - d(s_{j,u_1^b},\,s_{j,u_2^b})| < \epsilon,
\end{equation}
for some distance function $d$.
Note that the two distances in \eqref{eq:consistency} are between objects \emph{internal} to scans $\S_i$ and $\S_j$, respectively.
If a pair of correspondences are deemed consistent, the corresponding edge is attributed the weight $w(u_1,u_2):=f(c_{u_1,u_2})$, for some choice of ${f:\mathbb{R}_+\to[0,1]}$ that scores very consistent pairs close to 1.
In this paper, we choose $f(c):=\exp(-c^2/2\sigma^2)$ for simplicity, though other appropriate functions could be used.
Given a consistency graph, correspondences are selected that maximize consistency, further explained in Section~\ref{sec:clipper}.

The distance function $d$ must be carefully chosen to ensure accuracy of graph-based data association.
In particular, we desire \eqref{eq:consistency} to hold when $s_{j,u_1^b},\,s_{j,u_2^b}$ are the transformed versions of $s_{i,u_1^a},\,s_{i,u_2^a}$, respectively.
This property is called invariance and leads to subgraphs of the consistency graph that indicate a set of object matches in a loop pair.
\begin{definition}\label{defn:invariance}
A distance $d:X\times X\to\mathbb{R}$ is \emph{invariant} if $d(x_1,x_2) = d(x_1',x_2')$, where $x_1',x_2'\in X$ are the transformation of $x_1,x_2\in X$ under $T\in\mathrm{SE}(3)$, respectively.
\end{definition}
We establish the invariance of the metric $d_\Graff$ to rotation and, under careful application, translation.
\begin{prop}\label{prop:invariance}
For elements $\bY_1\in\Graff(k_1,3)$, $\bY_2\in\Graff(k_2,3)$ with affine coordinates $[A_1,b_1]$ and $[A_2,b_2]$, the affine Grassmannian metric $d_\Graff$ is invariant if the affine components are first \emph{shifted} to the origin, i.e., if both $b_1$ and $b_2$ are first translated by $-b_1$.
\end{prop}
\begin{proof}
See Appendix A.
\end{proof}
The intuition of Proposition~\ref{prop:invariance} can be understood from Fig.~\ref{fig:graffexample-both}.
As $\bY_1$ and $\bY_2$ are together translated further from the origin, the principal angle between $j(\bY_1)$ and $j(\bY_2)$ decreases to zero in the limit.
However, the distance between the affine components of $\bY_1$ and $\bY_2$ remains the same, no matter the translation.
By first shifting the affine components, we remove the dependence of the absolute translation in the computation of the principal angle, while maintaining the dependence on the \emph{relative} translation between $\bY_1$ and $\bY_2$.

A remaining challenge is to address the insensitivity of $d_\Graff$ to the Euclidean distance between affine components of objects.
The yellow curve ($s=0$) in Fig.~\ref{fig:graffsensitivity} represents the principal angle between $\bY_1,\bY_2\in\Graff(0,1)$ after shifting them as per Proposition~\ref{prop:invariance}, as a function of the Euclidean distance between $\bY_1$ and $\bY_2$.
Observe that after a distance of approximately \SI{2}{\meter}, the curve flattens significantly as it asymptotes towards $\tfrac{\pi}{2}$.
This nonlinearity leads to poor discrimination between pairs of correspondences whose internal objects are far apart in the Euclidean sense.
To combat this when calculating pairwise affine Grassmannian distances, we first scale the affine component of each $\bY_i$ by a constant parameter $\rho$ so that the affine coordinates of $\bY_i$ become $[A_i,b_i/\rho]$.
The choice of $\rho$ depends on the average Euclidean distance between objects in the environment and its effect is to bring principal angles into the linear regime.
The selection of $\rho$ is discussed further in Section~\ref{sec:exp-scaling}.

With Proposition~\ref{prop:invariance} and the scaling parameter $\rho$ in hand, a consistency graph between objects in $\S_i$ and $\S_j$ can be constructed.
We establish initial correspondences between each object in $\S_i$ with each object of $\S_j$ so long as the objects are of the same dimensions $k$ (i.e., we do not consider lines associated to planes).
Given additional information such as color, scan intensity, planar patch area, or pole radius, this initial set of correspondences could be refined, but would rely on accurately segmenting lines and planes across potentially wide baselines.
While we restrict object correspondences to be of the same dimension, the machinery we have developed allows for computing the consistency of two correspondences whose internal pair of objects have differing dimension.
Evaluating the consistency of a correspondence pair in our affine Grassmannian framework is illustrated in Fig.~\ref{fig:consistencygraph}.

\subsection{Graph-based Global Data Association}\label{sec:clipper}

Given a consistency graph, the task of matching objects from two scans is reduced to identifying the densest clique of consistent correspondences, formalized as the problem
\begin{gather}\label{eq:densestclique}
\begin{array}{ll}
\underset{u \in \{0,1\}^m}{\text{maximize}} & \dfrac{u^\top  M \, u}{u^\top u}
\\
\text{subject to} & u_i \, u_j = 0  \quad \text{if}~ M(i,j)=0, ~ \forall_{i,j},
\end{array}
\end{gather}
where $M\in[0,1]^{m\times m}$ is the weighted adjacency matrix (i.e., from $w$ as defined in Section~\ref{sec:consistency}) with ones on the diagonal, and ${u\in\{0,1\}^m}$ indicates a consistent set of correspondences.
Note that we choose to maximize the \emph{density} of correspondences rather than the cardinality (e.g., maximum clique) as our previous work has found this objective to produce more accurate results~\cite{lusk2021clipper}.
Problem~\eqref{eq:densestclique} is NP-hard, therefore we solve a particular relaxation which yields high accuracy solutions via our efficient CLIPPER algorithm (see ~\cite{lusk2021clipper} for more details).

\subsection{Transformation Estimation}

Given pairwise correspondences between objects in $\S_i$ and $\S_j$, consider finding the best rigid transformation to simultaneously align matched lines and planes by solving the optimization problem
\begin{equation}
\min_{\substack{R\in\mathrm{SO}(3),\\ t\in\mathbb{R}^3}}
\sum_{i=1}^{p} \|\pi_i' - f_\pi(\pi_i,R,t)\|^2
+
\sum_{i=1}^{l} \|\ell_i' - f_\ell(\ell_i,R,t)\|^2.
\end{equation}
This problem can be solved in closed-form by first solving for the rotation via SVD, then solving for the translation via least squares, similar to Arun's method for point cloud registration~\cite{arun1987least}.
The benefit of using the line and plane geometry directly, as opposed to a point parameterization, is twofold.
First, it allows the use of the full information present in the infinite plane or line, i.e., distance from origin as well as orientation.
Second, it does not require assumptions about where the ``centroid'' of the plane or line is, which is undefined for infinite planes and lines and requires consistent segmentation of objects across scans.
Together, these benefits lead to a more accurate rigid transformation estimate when aligning line and plane features.
\begin{figure}[t]
    \centering
    \begin{subfigure}[b]{0.49\columnwidth}
        \includegraphics[trim = 0mm 0mm 0mm 0mm, clip, width=1\linewidth]{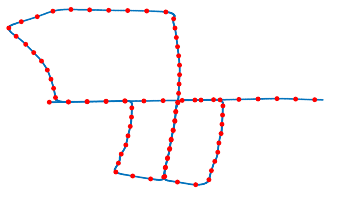}
        \caption{}
        \label{fig:kitti-traj-kfs}
    \end{subfigure}
    \begin{subfigure}[b]{0.49\columnwidth}
        \includegraphics[trim = 0mm 0mm 0mm 0mm, clip, width=1\linewidth]{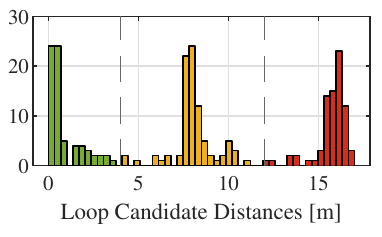}
        \caption{}
        \label{fig:candidate-dists}
    \end{subfigure}
    \caption{Candidate loop closure pair generation.
    (a) Keyframe poses of KITTI 05 are shown in red.
    (b) A histogram of candidate loop pair distances for all KITTI sequences used.
    Each keyframe is paired with three nearby previously visited poses of approximately \SI{0}{\meter}, \SI{8}{\meter}, and \SI{16}{\meter}.
    Candidate loops are grouped into $71$ easy pairs, $87$ medium pairs, and $69$ hard pairs. %
    }
    \label{fig:loop-candidate-generation}
\end{figure}

\section{Experiments}\label{sec:experiments}

We evaluate our global data association method using candidate loop closure pairs from KITTI~\cite{geiger2012we} sequences 00, 02, 05, and 08.
We compare our method, called GraffMatch, with a pole-only method~\cite{cao2021lidar} based on 2D cluster matching~\cite{cao2020accurate}, and a plane-only method~\cite{zhou2021pi} that attempts to match planes via nearest neighbor search on CP parameterization~\cite{geneva2018lips} followed by RANSAC~\cite{fischler1981random}.
We adapt the pole-only method~\cite{cao2021lidar} to 3D and denote it as PoleMatch, while the plane-only method is denoted PlaneMatch.
The algorithms are implemented in MATLAB\footnote{\href{https://github.com/mit-acl/clipper}{https://github.com/mit-acl/clipper}} and executed on an i9-7920X CPU with 64 GB RAM.

\subsection{Dataset}\label{sec:dataset}

Each sequence contains a trajectory of ground truth poses $T_i = (R_i,t_i)\in\mathrm{SE}(3)$.
We generate potential loop candidates by sampling $K$ keyframe poses $\bar{T}_k,\forall\,k\in[K]$ along the trajectory with a stride of \SI{20}{\m}, e.g., see Fig.~\ref{fig:kitti-traj-kfs}.
Let the set of all poses $T_i$ leading up to keyframe $\bar{T}_k$ be denoted $\mathcal{T}_k$.
The set of previously visited poses near keyframe $\bar{T}_k$ is then
\begin{equation*}
\mathcal{X}_k = \{ T_i\;\colon \|t_k-\bar{t}_i\| < r,\;\forall\,T_i\in\mathcal{T}_{k-1} \},
\end{equation*}
where we have set $r=\SI{20}{\m}$ to prevent selecting a loop pair without overlapping scans.
From each $\mathcal{X}_k\ne\emptyset$, three loop candidates are generated with $\bar{T}_k$ based on straight-line distance.
We used distances of \SI{0}{\m}, \SI{8}{\m}, and \SI{16}{\m}, for easy, medium, and hard difficulty, respectively.
These three cases allow us to evaluate each method's sensitivity to noise, baseline, and partial overlap.
Some keyframes may not have a loop candidate at a specified distance, resulting in an unequal number of easy, medium, and hard cases.
A histogram of these distances is shown in Fig.~\ref{fig:candidate-dists} for all KITTI sequences.

Pole and plane features are extracted from each loop candidate lidar scan and are used as input for each algorithm for global data association.
Poles are extracted as lines by leveraging the SemanticKITTI~\cite{behley2019semantickitti} dataset for simplicity.
Given points corresponding to the pole or trunk classes, we use DBSCAN~\cite{ester1996density} implemented in Open3D~\cite{Zhou2018} to generate clusters, from which PCA~\cite{shlens2014tutorial} is used to estimate a line.
Planar patches are extracted from the lidar scan using our implementation\footnote{\href{https://github.com/plusk01/pointcloud-plane-segmentation}{https://github.com/plusk01/pointcloud-plane-segmentation}} of~\cite{araujo2020robust}.
Because planar patches are bounded, there may be multiple planar patches that correspond to the same infinite plane.

\begin{figure}[t]
    \centering
    \includegraphics[trim=0cm 0cm 0cm 0cm, clip, width=\columnwidth]{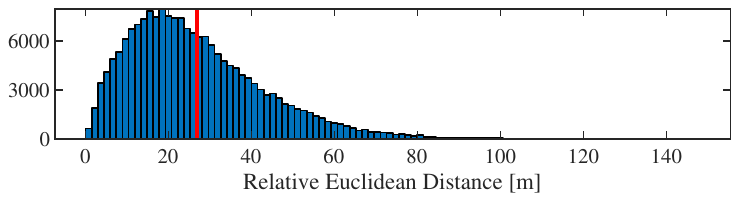}
    \caption{
    Pairwise object distances from KITTI 00, 02, 05, and 08.
    The mean is \SI[separate-uncertainty=true,multi-part-units=single]{27\pm16}{\meter}.
    Using this data, we choose the scaling parameter as $\rho=40$.
    }
    \vskip0.1in
    \label{fig:kitti-pdists}
\end{figure}

\subsection{Selection of Scaling Parameter}\label{sec:exp-scaling}

The scaling parameter $\rho$ (see Section~\ref{sec:consistency}) is chosen so that the pairwise affine Grassmannian distance lies in the linear regime and is therefore more sensitive when scoring consistencies.
The Velodyne HDL-64E used in KITTI has a range of \SIrange{50}{120}{\meter}, with an average point range in the KITTI dataset of approximately \SI{80}{\meter}.
In terms of pairwise object distances, we find that the average Euclidean distance is \SI[separate-uncertainty=true,multi-part-units=single]{27\pm16}{\meter}, as shown in Fig.~\ref{fig:kitti-pdists}.
Therefore, we select $\rho=40$ so that relative Euclidean distances of \SI{80}{\meter} will be scaled to \SI{2}{\meter}, which is at the end of the linear regime (see Fig.~\ref{fig:graffsensitivity}).

\begin{table}[!t] %
\centering
\caption{
Recall at \SI{100}{\percent} precision.
Divided into easy (E), medium (M), hard (H) cases based on straight-line distance between loop candidate poses.
}
\setlength{\tabcolsep}{3.3pt}
\ra{1.2}
\begin{tabular}{c c c c c c c c c c c c}
\toprule
Seq.  & \multicolumn{3}{c}{GraffMatch (Ours)} && \multicolumn{3}{c}{PoleMatch} && \multicolumn{3}{c}{PlaneMatch} \\
\cmidrule{2-4}\cmidrule{6-8}\cmidrule{10-12}
      & E & M & H   && E & M & H   && E & M & H \\
\toprule
$00$  &  \SI{91}{\percent} & \SI{78}{\percent} & \SI{46}{\percent} &&  \SI{69}{\percent} & \SI{43}{\percent} & \SI{41}{\percent} && \SI{66}{\percent} & \SI{3}{\percent} & \SI{3}{\percent} \\
$02$  &  \SI{100}{\percent} & \SI{78}{\percent} & \SI{50}{\percent} &&  \SI{44}{\percent} & \SI{33}{\percent} & \SI{17}{\percent} && \SI{33}{\percent} & \SI{11}{\percent} & \SI{0}{\percent} \\
$05$  &  \SI{95}{\percent} & \SI{68}{\percent} & \SI{35}{\percent} &&  \SI{42}{\percent} & \SI{41}{\percent} & \SI{18}{\percent} && \SI{42}{\percent} & \SI{14}{\percent} & \SI{6}{\percent} \\
$08$  &  \SI{100}{\percent} & \SI{79}{\percent} & \SI{78}{\percent} &&  \SI{55}{\percent} & \SI{32}{\percent} & \SI{44}{\percent} && \SI{0}{\percent} & \SI{0}{\percent} & \SI{0}{\percent} \\
\midrule
all  &  \SI{94}{\percent} & \SI{76}{\percent} & \SI{48}{\percent} &&  \SI{56}{\percent} & \SI{39}{\percent} & \SI{33}{\percent} && \SI{45}{\percent} & \SI{6}{\percent} & \SI{3}{\percent} \\
\bottomrule  
\end{tabular}
\label{tbl:recall}
\end{table}

\begin{table}[t] %
\centering
\caption{
Median translation and rotation alignment error of all successful loop closures, divided into easy (E), medium (M), hard (H) cases.
}
\setlength{\tabcolsep}{3.1pt}
\ra{1.2}
\begin{tabular}{c c c c c c c c c c c c}
\toprule
  & \multicolumn{3}{c}{GraffMatch (Ours)} && \multicolumn{3}{c}{PoleMatch} && \multicolumn{3}{c}{PlaneMatch} \\
\cmidrule{2-4}\cmidrule{6-8}\cmidrule{10-12}
      & E & M & H   && E & M & H   && E & M & H \\
\toprule
$\tilde{t}_\text{err}$ [cm]  &  $9.1$ & $17.3$ & $25.7$ && $10.4$ & $23.2$ & $16.0$ && $11.8$ & $17.3$ & $25.1$ \\
$\tilde{\theta}_\text{err}$ [deg]  &  $0.57$ & $0.92$ & $1.32$ && $0.74$ & $1.6$ & $1.72$ && $0.97$ & $1.78$ & $2.58$ \\
\bottomrule  
\end{tabular}
\label{tbl:alignment-error}
\end{table}

\subsection{Loop Closure Results}
Global data association is attempted on each loop closure candidate, after which line and plane matches are used to estimate a rigid transformation $\hat{T}^i_j$ of scan $j$ w.r.t scan $i$.
The quality of loop closure is evaluated by comparing $\hat{T}^i_j$ with the ground truth $T^i_j$ and calculating the rotation and translation error.
If the rotation error is less than \SI{5}{\degree} and the translation error is less than \SI{1}{\meter}, the loop closure is accepted.
If the number of matches returned by an algorithm is less than 3, the loop closure attempt is considered failed.
The parameters used for GraffMatch (see \eqref{eq:consistency}) are $\epsilon=0.2$ and $\sigma=0.02$.

Table~\ref{tbl:recall} lists the recall at \SI{100}{\percent} precision for each tested KITTI sequence.
As expected, utilizing both poles and planes in GraffMatch produces a higher number of successful loop closures.
The number of successful PoleMatch loop closures is low due to too few poles or variation of extracted poles across lidar scans (i.e., in a single scan, few lidar point returns may exist for a pole-like object, leading to a noisy centroid).
PlaneMatch also scores low in general and even fails to successfully match and align planes in all of sequence 08, where the car drives through previously visited streets in the opposite direction.
Because the CP parameterization heavily depends on the origin and orientation of the lidar sensor frame, successful CP plane matching requires a very good initialization, as in the easy case where PlaneMatch performs at its best.
This requirement can be problematic in the presence of odometry-only measurements, as drift could prevent loop closure from ever succeeding.

\begin{figure}[t]
    \centering
    \includegraphics[trim=0cm 0cm 0cm 0cm, clip, width=\columnwidth]{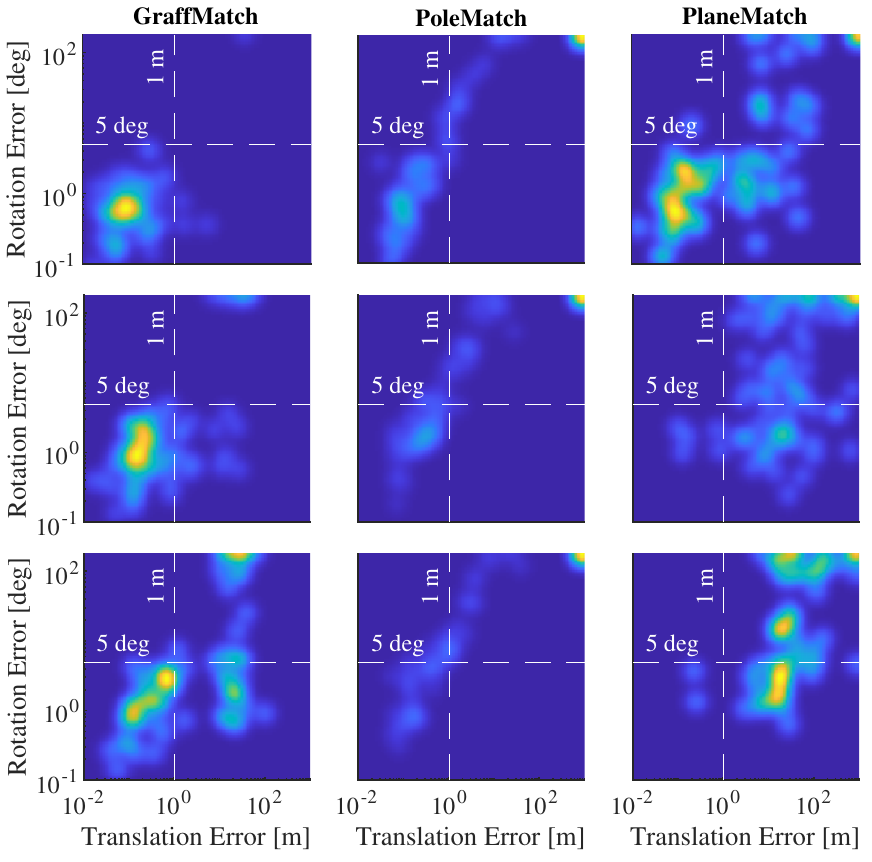}
    \caption{
    Alignment error for loop closure pairs, visualized as a grid of likelihood-normalized density plots.
    From left to right, the grid columns correspond to GraffMatch (ours), PoleMatch, and PlaneMatch.
    From top to bottom, the grid rows correspond to the easy, medium, and hard cases.
    For each case, GraffMatch achieves the highest recall, indicated by the high density of points in the low-translation, low-rotation error regime.
    PoleMatch fails to generate enough pole correspondences in many loop closures due to the scarcity of poles; in these cases, the error is set to a high value (upper-right corner).
    PlaneMatch performs at its best in the easy case when lidar scans have very close initial poses, but breaks down as the baseline distance increases.
    }
    \label{fig:kitti-alignment-err}
\end{figure}

Fig.~\ref{fig:kitti-alignment-err} shows the alignment error of loop candidates from all sequences as a $3\times3$ grid of density heatmaps, where columns correspond to algorithms and rows (from top to bottom) correspond to easy, medium, and hard cases.
In many cases of PoleMatch and in some cases of PlaneMatch, less than 3 matches were returned and so alignment error is set high, causing increased density in the upper-right corner.
GraffMatch is the only data association method that consistently scores in the low-translation, low-rotation error regime.
The median alignment error for successful loop closures is listed in Table~\ref{tbl:alignment-error}.

As discussed in Section~\ref{sec:consistency}, the distance function used to score consistency in our graph-theoretic framework is an important consideration.
We choose $d_\Graff$ because it allows us to score the consistency of affine subspaces pairs with arbitrary dimension in a principled manner.
Other distance functions might only consider the distance or angle between objects, for example.
Fig.~\ref{fig:recall-vs-distance} shows recall at \SI{100}{\percent} precision and compares our choice of $d_\Graff$ with four other possible distances.
The distances $d_\Gr$ and $d_{\pi\ell}$ disregard distance information, treating lines and planes as linear subspaces containing the origin, or naively using the inner product between a plane's normal vector and a line's direction vector, respectively.
The standard Euclidean distance $d_{\mathbb{R}^n}$ disregards subspace orientation and instead treats lines and planes as bounded, using their centroids as measurements.
As discussed previously in this section, using centroid requires that points be segmented into the same bounded lines and planes in every view, and thus will suffer as the baseline between loop pairs increases.
Naively combining orientation and distance information in $d_{\Gr\times\mathbb{R}^n}$ leverages all available information, but requires the weighting function $f$ (see Section~\ref{sec:consistency}) to take on an ad-hoc mixture of kernels with additional parameters, e.g., $f(c_r,c_\theta):=\exp(-c_r^2/\sigma_r^2)\exp(-c_\theta^2/\sigma_\theta^2)$.
Using $d_\Graff$ leads to a simple method of calculating distances on the manifold of affine subspaces and leads to higher recall.

\begin{figure}[t]
    \centering
    \includegraphics[trim=0cm 0cm 0cm 0cm, clip, width=\columnwidth]{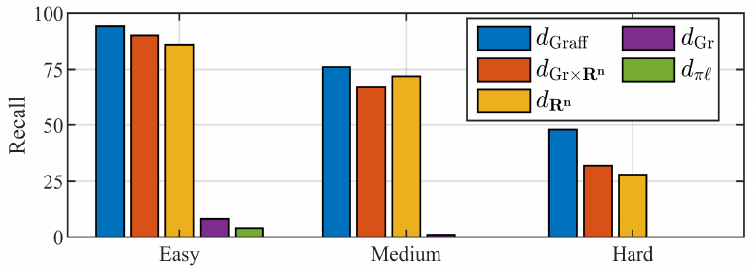}
    \caption{
    Recall at \SI{100}{\percent} precision of loop candidate alignment using different distance functions in our data association framework.
    The shifted affine Grassmannian distance $d_\Graff$, which combines line and plane `direction' with distance, provides the highest recall.
    Using centroid information ($d_\mathbb{R}^n$, $d_{\Gr\times\mathbb{R}^n}$) also gives good results, but depends on accurate line and plane segmentation.
    Using only directional information ($d_\Gr$, $d_{\pi\ell}$) performs poorly due to many objects with similar plane normals and line directions.
    }
    \label{fig:recall-vs-distance}
\end{figure}

Timing results for GraffMatch, PoleMatch, and PlaneMatch are respectively \SI[separate-uncertainty=true,multi-part-units=single]{0.076\pm0.102}{\second}, \SI[separate-uncertainty=true,multi-part-units=single]{0.005\pm0.004}{\second}, and \SI[separate-uncertainty=true,multi-part-units=single]{0.011\pm0.003}{\second}.
Thus, GraffMatch is suitable for online operation in loop closure tasks and is a robust alternative to PoleMatch and PlaneMatch, both of which rely on assumptions to speed up their execution, but limit their accuracy.
Specifically, PoleMatch treats infinite lines as centroid points and PlaneMatch requires an initial frame alignment guess.
In our experiments, there were an average of $7$ poles and $23$ planar patches extracted per frame, resulting in an average of $650$ initial correspondences to be processed for geometric consistency.
Execution time could be reduced by leveraging additional object information to immediately discard initial matches instead of allowing each object be potentially associated with each other object (e.g., a plane with large area is unlikely to be matched to a small plane).

\section{Conclusion}\label{sec:conclusion}

We presented a global data association method that achieved high recall with low alignment error when evaluated on candidate loop closures in the KITTI dataset.
By unifying the representation of poles and planes extracted from lidar scans as affine Grassmannian manifold elements, GraffMatch widens the applicability of using geometric primitive in place of memory-intensive point cloud maps.
Importantly, leveraging the invariant shifted affine Grassmannian distance in our graph-based data association framework enables the geometric verification of place recognition candidates with a wide range of baseline distances between frames.
By removing assumptions on initial frame alignment (e.g., from noisy odometry), GraffMatch is applicable to other perception problems requiring geometric verification, such as extrinsic sensor calibration, map merging, and global relocalization.

In future work we will incorporate GraffMatch into a complete SLAM pipeline, using affine Grassmannian objects for both local and global data association.
In particular, we will investigate the estimation of lines and planes directly via subspace tracking methods, using manifold-based optimization techniques to perform online bundle adjustment of affine Grassmannian object landmarks.
\appendix

\subsection{Proof of Invariance}

\begin{repprop}{prop:invariance}
For elements $\bY_1\in\Graff(k_1,3)$, $\bY_2\in\Graff(k_2,3)$ with affine coordinates $[A_1,b_1]$ and $[A_2,b_2]$, the affine Grassmannian metric $d_\Graff$ is invariant if the affine components are first \emph{shifted} to the origin, i.e., if both $b_1$ and $b_2$ are first translated by $-b_1$.
\end{repprop}
\begin{proof}
The subspace distance between $\bY_1$ and $\bY_2$ is $d_\Graff(\bY_1,\bY_2) = \|\Theta\|$, where $\Theta$ is a vector of $k=\min(k_1,k_2)$ principal angles.
These angles can be calculated via the singular value decomposition of $Y_1^\top Y_2$, the inner product of the Stiefel coordinates of $\bY_1,\bY_2$.
Without loss of generality, assume $\bY_1,\bY_2$ are shifted s.t. $b_{1}=0$.
Then,
\begin{equation}
Y_1^\top Y_2 =
\begin{bmatrix}
A_1^\top A_2 & \tfrac{1}{\eta_2}A_1^\top b_{02} \\
0 & \tfrac{1}{\eta_1\eta_2}
\end{bmatrix},
\end{equation}
where $\eta_i=\sqrt{\|b_{0i}\|^2 + 1}$.
Given $T=(R,t)\in\mathrm{SE}(3)$, let $\bar{\bY}_1,\bar{\bY}_2$ be the rotated and translated versions of $\bY_1,\bY_2$, respectively, with affine coordinates
\begin{equation}
\bY_i:[A_i,b_i] \xrightarrow{\quad T\quad} \bar{\bY}_i:[RA_i, Rb_i + t].
\end{equation}
Shifting $\bar{\bY}_1,\bar{\bY}_2$ by $-\bar{b}_1=-(Rb_1+t)$ leads to the affine coordinates $\bar{\bY}_1:[RA_1, 0]$ and
\begin{equation}
\bar{\bY}_2:[RA_2, Rb_2+t-(Rb_1+t)] = [RA_2, Rb_2],
\end{equation}
so that
\begin{equation}
\bar{Y}_1^\top \bar{Y}_2 =
\begin{bmatrix}
A_1^\top A_2 & \tfrac{1}{\eta_2}A_1^\top b_{02} \\
0 & \tfrac{1}{\eta_1\eta_2}
\end{bmatrix},
\end{equation}
which is free of $R$ and $t$ and equal to $Y_1^\top Y_2$, as desired.

\end{proof}

\balance %

\bibliographystyle{IEEEtran}
\bibliography{refs}

\end{document}